\documentclass{article}
\usepackage{ijcai24}

\usepackage{times}
\usepackage{soul}
\usepackage{url}
\usepackage{hyperref}
\usepackage[utf8]{inputenc}
\usepackage[small]{caption}
\usepackage{graphicx}
\usepackage{amsmath}
\usepackage{amsthm}
\usepackage{bm}
\usepackage{booktabs}
\usepackage{algorithm}
\usepackage{algorithmic}
\usepackage{amssymb}
\usepackage[switch]{lineno}
\usepackage{xcolor}
\usepackage{multirow}
\usepackage{makecell}
\usepackage{subfigure}

\newtheorem{definition}{Definition}
\newtheorem{theorem}{Theorem}
\newtheorem{proposition}{Proposition}
\newtheorem{appendixproposition}{Proposition}

\title{Deep Frequency Derivative Learning for Non-stationary
Time Series Forecasting }

\author{
Wei Fan$^1$
\and
Kun Yi$^2$\footnotemark[1]\and
Hangting Ye$^{3}$\and
Zhiyuan Ning$^{4}$\and
Qi Zhang$^5$\and
Ning An$^6$
\affiliations
$^1$University of Oxford
$^2$Beijing Institute of Technology
$^3$Jilin University\\
$^4$Chinese Academy of Science
$^5$Tongji University
$^6$Hefei University of Technology\\
\emails
wei.fan@wrh.ox.ac.uk,
yikun@bit.edu.cn,
yeht2118@mails.jlu.edu.cn,\\
ningzhiyuan@cnic.cn,
zhangqi\_cs@tongji.edu.cn,
ning.g.an@acm.org
}

\begin{document}

\maketitle

\renewcommand{\thefootnote}{\fnsymbol{footnote}}
\footnotetext[1]{Corresponding Author.}
\setcounter{footnote}{0}
\renewcommand{\thefootnote}{\arabic{footnote}}

\begin{abstract}
While most time series are non-stationary, it is inevitable for models to face the distribution shift issue in time series forecasting. Existing solutions manipulate statistical measures (usually mean and std.) to adjust time series distribution. However, these operations can be theoretically seen as the transformation towards zero frequency component of the spectrum which cannot reveal full distribution information and would further lead to \textit{information utilization bottleneck} in normalization, thus hindering forecasting performance. To address this problem, we propose to utilize \textit{the whole frequency spectrum} to transform time series to make full use of data distribution from the frequency perspective. We present a \textit{deep frequency derivative learning} framework, \textsc{DeRiTS}, for non-stationary time series forecasting. Specifically, \textsc{DeRiTS} is built upon a novel reversible transformation, namely \textit{Frequency Derivative Transformation} (FDT) that makes signals \textit{derived} in the frequency domain to acquire more stationary frequency representations. Then, we propose the  \textit{Order-adaptive Fourier Convolution Network} to conduct adaptive frequency filtering and learning. Furthermore, we organize \textsc{DeRiTS} as a \textit{parallel-stacked} architecture for the multi-order derivation and fusion for forecasting. 
Finally, we conduct extensive experiments on several datasets which show the consistent superiority in both time series forecasting and shift alleviation.

\end{abstract}

\section{Introduction} \label{sec:intro}

Time series forecasting has been playing an important role in a variety of real-world industries, such as traffic analysis~\cite{ben1998dynamit}, weather prediction~\cite{lorenz1956empirical}, financial estimation~\cite{king1966market,ariyo2014stock}, energy planning~\cite{fan2024dewp}, etc. 
Following by classic
statistical methods (e.g., ARIMA~\cite{whittle1963prediction}), many deep machine learning-based time series forecasting methods~\cite{salinas2020deepar,han2024bigst,zhang2023irregular} have recently achieved superior performance in different scenarios. 
Despite the remarkable success, the non-stationarity widely existing in time series data has still been a critical but under-addressed challenge for accurate forecasting~\cite{priestley1969test,huang1998empirical,brockwell2009time}.

\begin{figure}
\centering
\subfigure[Transformation with only the  zero frequency component.]{\includegraphics[width=0.45\linewidth]{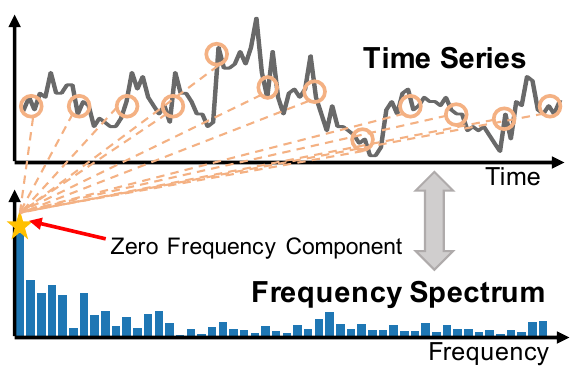} \label{fig:1a} } 
\subfigure[Transformation using the  whole frequency spectrum.]{ \includegraphics[width=0.45\linewidth]{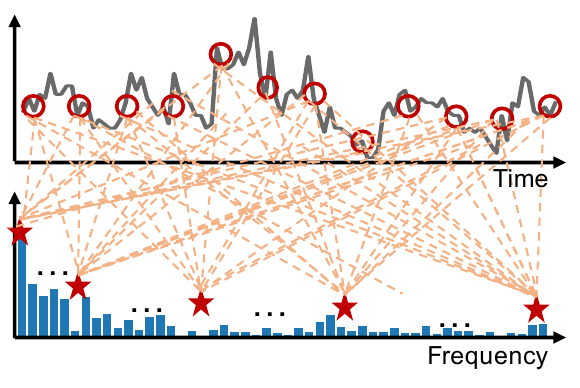} \label{fig:1b} }
\vspace{-5mm}
\caption{Given one time series and its frequency spectrum, the main comparison between existing works (a) and our method (b).}
\label{fig:motivation}
\vspace{-4mm}
\end{figure}

Since time series data are usually collected at a high frequency over a long duration, such non-stationary sequences with millions of timesteps inevitably let forecasting models face the distribution shifts over time. This would lead to performance degradation at test time~\cite{kim2021reversible} due to the covariate shift or the conditional shift~\cite{woo2022deeptime}. For this issue, pioneer works~\cite{ogasawara2010adaptive} propose to normalize time series data with global statistics; one recent work \cite{kim2021reversible} proposes to use instance statistics to normalize time series against distribution shifts. Then, some work brings statistical information into self-attention computation~\cite{liu2022non}; another work~\cite{fan2023dish}  transform time series with learnable statistics and consider shifts input and output sequences, and \cite{liu2023adaptive} utilize predicted sliced statistics for adaptive normalization. 

Most of these existing works focus on transforming each timestep of time series with certain statistics (usually mean and std.) After our careful theoretical analysis\footnote{We leave the details of our theoretical analysis in Appendix \ref{proof_mean}}, we have found that these operations can actually be regarded as the \textit{normalization towards the zero frequency component of the spectrum} in the frequency domain, as shown in Figure \ref{fig:1a}. However, they cannot fully utilize distribution information of time series signals; moreover, this would lead to {\textit{information utilization bottleneck}} in normalization and thus hinders the performance of time series forecasting. 
To address this problem, we propose to utilize \textit{the whole frequency spectrum} for the transformation of time series, to make full use of distribution information of time series from the frequency perspective and in the meanwhile transform time series into more stationary space thus making more accurate forecasting.
Figure \ref{fig:1b} has shown our method with whole frequency spectrum. 



Motivated by this view, we then present a \textit{deep frequency derivative learning} framework, \textsc{DeRiTS}, for non-stationary time series forecasting.
The core idea of \textsc{DeRiTS} lies in two folds: (i) employing the whole frequency spectrum to take the derivative of time series signals, and (ii) learning frequency dependencies on more stationary transformed representations.
Specifically, we first propose a novel transformation for time series signals in \textsc{DeRiTS}, namely \textit{Frequency Derivative Transformation} (FDT), which mainly includes two stages.
In the first stage, the raw signals in the time domain are transformed into the frequency domain with Fourier transform~\cite{nussbaumer1982fast} for further learning. 
In the second stage, the transformed frequency components are derived with respect to timestamps to get more stationary frequency representations.
Inspired by the derivative in mathematics~\cite{hirsa2013introduction}, FDT let models aim for modeling gradients of signals rather than raw input signals, which could mitigate their burden of forecasting with distribution shifts by resolving non-stationary factors  (e.g., the shift of trends) in the original time series through one- or high-order derivation.

After acquiring more stationary representations, we further propose a novel architecture, \textit{Order-adaptive Fourier Convolution Network} (OFCN) in \textsc{DeRiTS} for the frequency filtering and dependency learning to accomplish the forecasting. Concretely, OFCN is composed of (i) Order-adaptive frequency filter that adaptively extracts meaningful patterns by excluding high-frequency noises for derived signals of different orders, and (ii) Fourier convolutions that conduct dependency mappings and learning for complex values in the frequency domain.
Since OFCN is operating in the projection space by FDT, we thus utilize the \textit{{inverse} Frequency Derivative Transformation} to recover the predicted frequency components back to the original time domain. Inspired by previous work~\cite{kim2022reversible}, we let all stages of FDT fully reversible and symmetric and make OFCN predict in the more stationary frequency space, which reveals our superiority in enhancing forecasting against distribution shifts.
Furthermore, in order for the \textit{multi-order} derivative learning, we have organized \textsc{DeRiTS} as a \textit{parallel-stacked} architecture to fuse representations of different orders. Specifically, \textsc{DeRiTS} is composed of several parallel branches, each of which represents an order of derivation and prediction corresponding for its FDT and OFCN. Note that the Fourier convolution adapted in each branch is not parameter-sharing and after the distinct processing of different branches, the outputs are fused to achieve the final time series forecasting.
In summary, our main contribution can be listed as follows:
\begin{itemize}
    \item Motivated by our theoretical analysis towards existing time series normalization techniques from the frequency spectrum perspective, we propose to utilize the \textit{whole frequency spectrum} for the transformation of time series. 
    \item We present a \textit{deep frequency derivative learning} framework, namely \textsc{DeRiTS}, built upon our proposed \textit{Frequency Derivative Transformation} (with its inverse)  for non-stationary time series forecasting.
    
    \item We introduce the novel \textit{Order-adaptive Fourier Convolution Network}, for the frequency dependency learning and organize \textsc{DeRiTS} as a \textit{parallel-stacked} architecture to fuse \textit{multi-order} representations for forecasting.

    \item We have conducted extensive experiments on seven real-world datasets, which have demonstrated the consistent superiority compared with state-of-the-art methods in both time series forecasting and shift alleviation.

\end{itemize}

\section{Related Work}
\subsection{Time Series Forecasting with Non-stationarity}
Time series forecasting is a longstanding research topic. Traditionally, researchers have proposed statistical approaches, including exponentially weighted moving averages~\cite{holt1957forecasting} and ARMA~\cite{whittle1951hypothesis}.
Recently, with the advanced development of deep learning~\cite{chen2023tabcaps,fan2020autofs,10187687,ning2021lightcake,chen2024excelformer,fan2021interactive,pu2022meta,ning2022graph}, many deep time series forecasting methods have been developed, including RNN-based methods (e.g., deepAR~\cite{salinas2020deepar}, LSTNet~\cite{Lai2018}), CNN-based methods (e.g., SCINet~\cite{liu2022scinet}, TCN~\cite{bai2018}), MLP-based Methods (e.g., DLinear~\cite{zeng2022transformers}, N-BEATS~\cite{Oreshkin2020N-BEATS}) and Transformer-based methods (e.g., Autoformer~\cite{autoformer21}, PatchTST~\cite{PatchTST2023}). While time series are non-stationary, existing works try normalize time series with global statistics~\cite{ogasawara2010adaptive}, instance statistics~\cite{kim2021reversible}, learnable statistics~\cite{fan2023dish} and sliced statistics~\cite{liu2023adaptive} in order to relieve the influence of distribution shift on forecasting. Other works bring time-index information~\cite{woo2022deeptime} or statistical information into network architectures~\cite{liu2022non,fan2024addressing} to overcome the shifts.



\subsection{Frequency Analysis in Time Series Modeling}
The frequency analysis has been widely used to extract knowledge of the frequency domain in time series modeling and forecasting. Specifically, SFM \cite{ZhangAQ17} adopts Discrete Fourier Transform to decomposes the hidden state of time series by LSTM into frequency components;  StemGNN~\cite{Cao2020} adopts Graph Fourier Transform to perform graph convolutions and uses Discrete Fourier Transform to computes series-wise correlations. Autoformer~\cite{autoformer21} replaces self-attention in Transformer~\cite{vaswani2017attention} and proposes the auto-correlation mechanism implemented by Fast Fourier Transform. FEDformer~\cite{fedformer22} introduces Discrete Fourier Transform-based frequency enhanced attention by acquiring the attentive weights by frequency components and then computing the weighted sum in the frequency domain. In addition, \cite{cost22} transforms hidden features of time series into the frequency domain with Discrete Fourier Transform; \cite{depts_2022} uses Discrete Cosine Transform to extract periodic information; \cite{frets_2023} combines Fast Fourier Transform (FFT) with MLPs; \cite{yi2024fouriergnn} combines FFT and graph neural network for time series forecasting~\cite{yi2023survey}.

\section{Problem Formulation}

\paragraph{Time Series Forecasting}
Let $\bm{x} = [ \bm{x}_1; \bm{x}_2; \cdots; \bm{x}_T ] \in \mathbb{R}^{T \times D}$ be regularly sampled multi-variate time series with $T$ timestamps and $D$ variates, where $\bm{x}_t \in \mathbb{R}^D$ denotes the multi-variate values at timestamp $t$. In the task of time series forecasting,
we use $\mathbf{X}_t \in \mathbb{R}^{L \times D}$ to denote the lookback window, a length-$L$ segment of $\bm{x}$ ending at timestamp $t$ (exclusive), namely $\mathbf{X}_t = \bm{x}_{t-L:t} = [\bm{x}_{t-L}; \bm{x}_{t-L+1}; \cdots; \bm{x}_{t-1}]$.
Similarly, we represent the horizon window as a length-$H$ segment of $\bm{x}$ starting from timestamp $t$ (inclusive) as $\mathbf{Y}_t$, so we have $\mathbf{Y}_t = \bm{x}_{t:t+H} = [\bm{x}_{t}; \bm{x}_{t+1}; \cdots; \bm{x}_{t+H-1}]$.
The classic time series forecasting formulation is to project
lookback values $\mathbf{X}_t$ into horizon values $\mathbf{Y}_t $.
Specifically, a typical forecasting model $F_\theta : \mathbb{R}^{L \times D} \to \mathbb{R}^{H \times D}$ produces forecasts by $ \hat{\mathbf{Y}}_t  = f_\theta( \mathbf{X}_t  )$ where $\hat{\mathbf{Y}}_t$ stands for the forecasting results and $\theta$ encapsulates the model parameters.

\paragraph{Non-stationarity and Distribution Shifts}
In this paper, we aim to study the problem of non-stationarity in deep time series forecasting. As aforementioned in Section \ref{sec:intro}, long time series with millions of timesteps let forecasting models face distribution shifts over time due to the non-stationarity. The distribution shifts in time series forecasting are usually the covariate shift~\cite{wiles2021fine,woo2022deeptime}. Specifically,
given a stochastic process, let $p\left(x_t, x_{t-1}, \ldots, x_{t-L+1}\right)$ be the unconditional joint distribution of a length $L$ segment where $x_t$ is the value of univariate time series at timestamp $t$.
The stochastic process experiences \textit{covariate shift} if any two segments are drawn from different distributions, i.e. 
$p\left( x_{t-L},x_{t-L+1} \ldots, x_{t-1} \right) \neq p\left( x_{t^{\prime}-L}, x_{t^{\prime}-L+1}, \ldots,  x_{t^{\prime}-1} \right),  \forall t \neq t^{\prime}$. 
Subsequently,
let $p\left(x_{t} \mid x_{t-1}, \ldots, x_{t-L}\right)$ represents the conditional distribution of $x_{t}$, such a stochastic process experiences \textit{conditional shift} if two segments have different conditional distributions, i.e. 
$p\left(x_{t} \mid x_{t-1}, \ldots, x_{t-L+1}, x_{t-L}\right) \neq$ $p\left(x_{t^{\prime}} \mid  x_{t^{\prime}-1}, \ldots, x_{t^{\prime}-L+1},x_{t^{\prime}-L}\right), \forall t \neq t^{\prime}$.

\section{Methodology}
In this section, we elaborate on our proposed \textit{deep frequency derivative learning} framework, \textsc{DeRiTS}, designed for non-stationary time series forecasting. First, we introduce our novel reversible transformation, \textit{Frequency Derivative Transformation} (FDT) in Section \ref{sec:method_fdt}. Then, to fuse multi-order information, we present the parallel-stacked frequency derivative learning architecture in Section \ref{sec:fre_arch}. Finally, we introduce our \textit{Order-adaptive Fourier Convolution Network} (OFCN) for frequency learning in Section \ref{sec:FCN}.

\begin{figure*}
    \centering
    \includegraphics[width=0.98\textwidth]{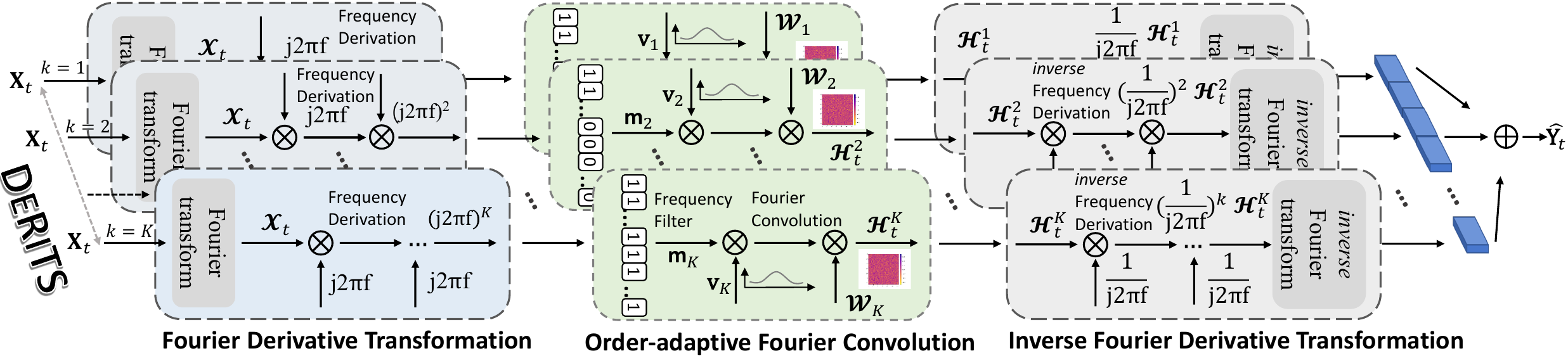}
    \caption{The main architecture of \textsc{DeRiTS}.}
    \label{fig:model}
\end{figure*}

\subsection{Frequency Derivative Transformation} \label{sec:method_fdt}
As aforementioned in Section \ref{sec:intro}, to fully utilize the whole frequency spectrum for the transformation of time series with sufficient distribution information, we propose the novel \textit{Frequency Derivative Transformation} (FDT) to achieve more stationary frequency representations of time series signals. 
For this aim, FDT mainly includes two distinct stages respectively for domain transformation and frequency derivation.  

\subsubsection{Domain Transformation} In the first stage, to be specific, we make use of fast Fourier transform~\cite{nussbaumer1982fast} to enable the decomposition of time series signals from the time domain into their inherent frequency components. Formally, given the time domain input signals $X(t)$, we convert it into the frequency domain by: 
\begin{equation}\label{fft}
\begin{aligned}
 \mathcal{X}(f)=&\mathcal{F}(X(t)) = \int_{-\infty }^{\infty} {X}(t) e^{-j2\pi ft}\mathrm{d}t \\
 =& \int_{-\infty }^{\infty} {X}(t) \cos(2\pi ft)\mathrm{d}t + j \int_{-\infty }^{\infty} {X}(t) \sin(2\pi ft)\mathrm{d}t,
 \end{aligned}
\end{equation}
where $\mathcal{F}$ is the fast Fourier transform, $f$ is the frequency variable, $t$ is the integral variable, and $j$ is the imaginary unit, defined as the square root of -1; 
$\int_{-\infty }^{\infty} {X}(t) \cos(2\pi ft)\mathrm{d}t$ is the real part of $\mathcal{X}$ and is abbreviated as $Re(\mathcal{X})$;  $\int_{-\infty }^{\infty} {X}(t) \sin(2\pi ft)\mathrm{d}t$ is the imaginary part and is abbreviated as $Im(\mathcal{X})$. After that we can rewrite $\mathcal{X}$ as $\mathcal{X}=Re(\mathcal{X})+jIm(\mathcal{X})$. 

\subsubsection{Frequency Derivation} In the second stage, with the transformed frequency components, we propose to utilize the whole frequency spectrum for the signal derivation, in order to represent time series in a more stationary space. The basic idea is to perform our proposed \textit{Fourier Derivative Operator} in the frequency domain, which is defined as follows: 
\begin{definition}[\textbf{Fourier Derivative Operator}]
    Given the time domain input signals $X(t)$ and its corresponding frequency components $\mathcal{X}(f)$, we then define $\mathcal{R}(\mathcal{X}(f)) := (j2\pi f)\mathcal{X}(f)$ as the {Fourier Derivative Operator} (FDO), where $f$ is the frequency variable and $j$ is the imaginary unit. 
\end{definition}
In the derivation, different order usually represents different signal representations. We propose to incorporate multi-order information in \textsc{DeRiTS} to further enhance the forecasting. For this aim, we extend above definition and further define the \textit{$k$-order Fourier Derivative Operator} $\mathcal{R}_k$ as:
\begin{equation}
    \mathcal{R}_k(\mathcal{X}(f)) =  (j2\pi f)^k\mathcal{X}(f).
\end{equation}
With such two stages, we can finally write the \textit{$k$-order Frequency Derivation Transformation} ${\rm FDT}_k$  as:
\begin{equation} \label{eq:FDT}
    {\rm \operatorname{FDT}}_k(X(t)) = (j2\pi f)^k \mathcal{F}(X(t))
\end{equation}
where $X(t)$ is the time domain input signal; $\mathcal{F}$ stands for fast Fourier transform and $f$ is the frequency variable.

\begin{proposition}
Given $X(t)$ in the time domain and $\mathcal{X}(f)$ in the frequency domain correspondingly,
the $k$-order Fourier Derivative Operator on $\mathcal{X}(f)$ is equivalent to $k$-order derivation on $X(t)$ with respect to $t$ in the time domain, written by:
\begin{equation}
    (j2\pi f)^k\mathcal{X}(f) = \mathcal{F}(\frac{\mathrm{d}^k X(t)}{\mathrm{d} t^k} ),
\end{equation}
where $\mathcal{F}$ is Fourier transform, $\frac{\mathrm{d}^k}{\mathrm{d} t^k}$ is $k$-order derivative with respect to $t$, and $j$ is the imaginary unit.
\end{proposition}
We leave the detailed proof in Appendix \ref{proof_fdt}. With such a equivalence, we can find out FDT can actually achieve more stationary representations in the lower order by derivation. For example, the shifts caused by a single trend signal in time series can nearly degraded be zero. We include the specific analysis in Appendix \ref{appendix:analysis}. Then, with less distribution shifts and non-stationarity by FDT, the deep networks can have large potential to perform more accurate forecasting.

\subsection{Frequency Derivative Learning Architecture} \label{sec:fre_arch}
The main architecture of \textsc{DeRiTS} is depicted in Figure \ref{fig:model}, which is built upon the Frequency Derivative Transformation and its inverse for the frequency derivative learning.



\paragraph{FDT/iFDT} As mentioned in Section~\ref{sec:method_fdt}, \textsc{DeRiTS} needs to conduct predictions in a more stationary frequency space achieved by frequency derivative transformation. We naturally need to recover the predictions back to the time domain for final forecasting and evaluation. To make FDT fully reversible, we let both stages of FDT reversible, including Fourier transform and Fourier Derivation. Specifically, following Equation (\ref{eq:FDT}), we can symmetrically write the \textit{inverse frequency derivative transformation} (iFDT) of $k$ order as:
\begin{equation}
  {\rm \operatorname{iFDT}}_k(\mathcal{X}(f)) = \mathcal{R}_k^{-1}( \mathcal{X}(f)) = \mathcal{F}^{-1}(\frac{1}{(j2\pi f)^k} (\mathcal{X}(f))),
\end{equation}
where $\mathcal{R}_k^{-1}$ is the inverse process of Fourier Derivative Operator of $k$ order; $\mathcal{F}^{-1}$ is the inverse Fourier transform; $\mathcal{X}(f)$ is the frequency components that need to be recovered to the time domain. Actually, the inverse process $\mathcal{R}_k^{-1}$ is equivalent to an integration operator in the time domain. More details can be found in Appendix~\ref{appendix:analysis}. 

\subsubsection{The Parallel-Stacked Architecture} To conduct the \textit{multi-order} frequency derivation transformation and learning, we have organized our \textsc{DeRiTS} framework as a \textit{parallel-stacked} architecture, where each branch represents an \textit{order} of frequency derivation learning, as shown in Figure \ref{fig:model}.
Let \textsc{DeRiTS} have $K$ branches in total. For each branch, we first take lookback values $\mathbf{X}_t$ to frequency derivative transformation by:
\begin{equation}
    \mathcal{X}_t^k = {\rm FDT}_k (\mathbf{X}_t), \;\; k = 1,2,\cdots, K
\end{equation}
where ${\rm FDT}_k $ is the $k$-order FDT and $\mathcal{X}_t^k$ is the frequency derivative representation for $\mathbf{X}_t^k$ at timestamp $t$.
Then, the learned representations for each branch are taken to the \textit{Fourier Convolution Network} (FCN) for frequency dependency learning. Since our FCN is order-adaptive in each parallel branch, we also take $k$ as input with the computation by:
\begin{equation}
 \mathcal{H}_t^k  = {\rm Order\text{-}adaptiveFourierConvolution}( k, \mathcal{X}_t^k)
\end{equation}
where $\mathcal{H}_t^k$ are the predicted frequency components for $\mathcal{X}_t^k$. Note that  ${\rm FourierConvolution}$ is not parameter-sharing for different branches. After that, we recover the predictions to the time domain by:
\begin{equation}
    \mathbf{H}_t^k = {\rm iFDT}_k ( \mathcal{H}_t^k ), \;\; k = 1,2,\cdots, K
\end{equation}
where $\mathbf{H}_t^k$ is the recovered representation of $k$ order in the time domain. After acquiring it, we finally fuse the multi-order representations from parallel branches for the forecasting with MLP layers, which is given by:
\begin{equation}
  \hat{\mathbf{Y}}_t = {\rm MultilayerPerceptron}(\mathbf{H}_t^1, \mathbf{H}_t^2, \cdots, \mathbf{H}_t^K)
\end{equation}
where $\hat{\mathbf{Y}}_t$ are forecasting results by \textsc{DeRiTS} for evaluation.



    

\begin{table*}[h]
    \centering
    \caption{Overall performance of time series forecasting. We set the lookback window size $L$ as 96 and vary the prediction length $H$ in $\{96, 192, 336, 720\}$; for traffic dataset, the prediction length $H$ is $ \{48, 96, 192, 336 \}$. The best results are in  \underline{bold} and the second best are  {underlined}. Full results of time series forecasting including ILI datasets are included in Appendix \ref{appendix:results} due to space limit.
    }
    \label{tab:long_term}
    \scalebox{0.78}{
    \begin{tabular}{l c|c c|c c|c c|c c| c c|c c |c c |c c}
    
    \toprule
   \multicolumn{2}{c|}{Models}  &\multicolumn{2}{c|}{\textsc{DeRiTS}} &\multicolumn{2}{c|}{FreTS} & \multicolumn{2}{c|}{PatchTST} & \multicolumn{2}{c|}{LTSF-Linear} & \multicolumn{2}{c|}{FEDformer} & \multicolumn{2}{c|}{Autoformer} & \multicolumn{2}{c|}{Informer} & \multicolumn{2}{c}{{NSTransformer}}\\
        \multicolumn{2}{c|}{Metrics}&MAE &RMSE &MAE &RMSE &MAE &RMSE&MAE &RMSE&MAE &RMSE &MAE &RMSE &MAE &RMSE &MAE &RMSE\\
       \midrule
        \multirow{4}*{\rotatebox{90}{Exchange}} & 96 & \textbf{0.035} & \textbf{0.050}& \underline{0.037} &  \underline{0.051}&0.039&  {0.052} & {0.038} &  {0.052}& 0.050& 0.067& 0.054& 0.070& 0.066& 0.084& 0.052 & 0.068\\
         & 192 & \textbf{0.050} & \textbf{0.066} &  \textbf{0.050}&  \underline{0.067}&0.055& 0.074 &  \underline{0.053}&  {0.069}& 0.060& 0.080& 0.065& 0.083& 0.068& 0.088& 0.062& 0.082\\
         & 336 & \textbf{0.060} & {0.083} &  \underline{0.062}&  \underline{0.082}&0.071& 0.093 &  {0.064}&  \textbf{0.080}& 0.070& 0.095& 0.085& 0.101& 0.093& 0.127&0.077 & 0.098\\
         & 720 & \textbf{0.086} & \textbf{0.108} &  \underline{0.088}&  \underline{0.110}&0.132& 0.166 &  {0.092}&  {0.116}& 0.142& 0.174& 0.150& 0.181& 0.117& 0.170& 0.140& 0.172\\
         \midrule
         \multirow{4}*{\rotatebox{90}{Weather}} &  96 & \textbf{0.030} & \textbf{0.070} &  \underline{0.032}&  \underline{0.071}&  {0.034}&  {0.074}& 0.040& 0.081&0.050 & 0.088&0.064 & 0.104& 0.101& 0.139& 0.055& 0.092\\
         & 192 & \textbf{0.037} & \textbf{0.078} &  \underline{0.040}&  \underline{0.081}&  {0.042}&  {0.084} & 0.048& 0.089& 0.051& 0.092& 0.061& 0.103&0.097 & 0.134& 0.057& 0.099\\
          &  336  & \textbf{0.042} & \textbf{0.090} &  \underline{0.046}&  \underline{0.093}&  {0.049}&  {0.094}& 0.056& 0.098& 0.057& 0.100&0.059 &0.101 &0.115 & 0.155& 0.056& 0.099\\
         & 720 & \textbf{0.050} & \textbf{0.094} & \underline{0.055} &  \underline{0.099}&  {0.056}&  {0.102}& 0.065& 0.106& 0.064 & 0.109&0.065 &0.110 &0.132 & 0.175& 0.063& 0.108\\
         \midrule
         \multirow{4}*{\rotatebox{90}{Traffic}} &  48 & {0.019} & {0.037} &  \underline{0.018}&  \underline{0.036}&  \textbf{0.016} &  \textbf{0.032} &0.020 & 0.039& 0.022& 0.036& 0.026& 0.042& 0.023& 0.039&  0.024&  0.038\\
         & 96 & \textbf{0.018} & \textbf{0.034} &  \underline{0.020}&  {0.038}&  \textbf{0.018}&  \underline{0.035} &0.022 & 0.042& 0.023& 0.044& 0.033& 0.050& 0.030& 0.047& 0.025& 0.046\\
          &  192 & \textbf{0.017} & \textbf{0.036} &  \underline{0.019}&  \underline{0.038}&  {0.020} &  {0.039} &  {0.020}& 0.040& 0.022& 0.042& 0.035& 0.053& 0.034& 0.053& 0.030& 0.048\\
         & 336 & \textbf{0.018} & \textbf{0.037} & \underline{0.020} &  \underline{0.039}&  {0.021} &  {0.040} &  {0.021}& 0.041& 0.021& 0.040& 0.032& 0.050& 0.035&0.054 & 0.031& 0.047\\
         \midrule
         \multirow{4}*{\rotatebox{90}{Electricity}} &  96 & \textbf{0.036} & \textbf{0.062} &  \underline{0.039}&  \underline{0.065}&  {0.041}&  {0.067}& 0.045& 0.075& 0.049&0.072 & 0.051& 0.075& 0.094 & 0.124& 0.050& 0.073\\
         & 192 & \textbf{0.038} & \underline{0.065} &  \underline{0.040}&  \textbf{0.064}&  {0.042}&  {0.066}& 0.043& 0.070& 0.049& 0.072& 0.072& 0.099& 0.105& 0.138&0.052 &0.080 \\
          &  336 & \textbf{0.041} & \underline{0.070} & {0.046}& {0.072}&  \underline{0.043}&  \textbf{0.067}&  {0.044}&  {0.071}& 0.051&0.075 & 0.084& 0.115& 0.112& 0.144& 0.064& 0.090\\
         & 720 & \textbf{0.048} & \textbf{0.076} & \underline{0.052} &  \underline{0.079}& 0.055& {0.081}&  {0.054}&  {0.080}&0.055 & {0.077}& 0.088& 0.119& 0.116& 0.148&0.068 & 0.094\\
         \midrule
         \multirow{4}*{\rotatebox{90}{ETTh1}} &  96 & \textbf{0.060} & \textbf{0.086} &  \underline{0.061}&  \underline{0.087}& 0.065& 0.091&  {0.063}&  {0.089}& 0.072&0.096 &0.079 &0.105 &0.093 & 0.121& 0.075& 0.098\\
         & 192 &\underline{0.066} &  \underline{0.093} &  \textbf{0.065}&  \textbf{0.091}& 0.069&  {0.094}&  {0.067}&  {0.094}& 0.076& 0.100&0.086 & 0.114& 0.103& 0.137&0.078 & 0.104\\
          &  336 & \textbf{0.068} & \textbf{0.095} &  \underline{0.070}&  \underline{0.096}&  {0.073} & {0.099} &  {0.075}&  {0.097}& 0.080& 0.105&0.088 & 0.119& 0.112& 0.145& 0.085& 0.109\\
         & 720 & \textbf{0.080}&  \textbf{0.107}& \underline{0.082} &  \underline{0.108}&  {0.087}&  {0.113}&  {0.083}& {0.110} & 0.090& 0.116&0.102 & 0.136& 0.125& 0.157& 0.096& 0.124\\
         \midrule
         \multirow{4}*{\rotatebox{90}{ETTm1}} &  96 & \textbf{0.050} & \textbf{0.075} &  \underline{0.052}&  \underline{0.077}&  {0.055} & 0.082&  {0.055}&  {0.080}& 0.063&0.087 &0.081 &0.109 &0.070 & 0.096 & 0.064& 0.087\\
         & 192 & \textbf{0.055} & \textbf{0.080} &  \underline{0.057}&  \underline{0.083}&  {0.059}&  {0.085}& {0.060}& {0.087}& 0.068& 0.093 &0.083 & 0.112& 0.082& 0.107& 0.070& 0.098\\
          &  336 & \textbf{0.060} & \textbf{0.086} &  \underline{0.062}&  \underline{0.089}&  {0.064} &  {0.091} & {0.065}& {0.093}& 0.075& 0.102&0.091 & 0.125& 0.090& 0.119& 0.079& 0.110\\
         & 720 & \textbf{0.064} & \textbf{0.094} & \underline{0.069} &  \underline{0.096}&  {0.070}&  {0.097}& {0.072}&{0.099} & 0.081& 0.108&0.093 & 0.126 & 0.115& 0.149& 0.086& 0.114\\
         \bottomrule
    \end{tabular}
    }

\end{table*}

\subsection{Order-adaptive Fourier Convolution Network} \label{sec:FCN}
Apart from the frequency derivative transformation, another aim of \textsc{DeRiTS} is to accomplish the dependency learning with the derived signals in the frequency domain. We thus introduce a novel network architecture, namely \textit{Order-adaptive Fourier Convolution Network} (OFCN) to enable the frequency learning. 
Specifically, OFCN is composed of two important components, i.e., order-adaptive frequency filter and Fourier convolutions, which are illustrated as follows:

\paragraph{Order-adaptive Frequency Filter} We aim to fuse  multi-order derived signal information for forecasting, while it is notable that different order corresponds to different frequency comments. 
Since time series include not only meaningful patterns but also high-frequency noises, we develop an order-adaptive frequency filter to enhance the learning process.

Supposing there are $S$ frequencies in $\mathcal{X}_t^k$, we sort $\mathcal{X}_t^k$ on the frequencies in a \textit{descending order} of amplitude for each frequency to get $\mathcal{X'}_t^k$ for $k$ order. Then,  we design an adaptive mask $\mathbf{m}_k$ to concentrate on only $\frac{S}{2^{(K-k)}}$ frequency components of $\mathcal{X}_t^k$ for further learning. We write the adaptive frequency filtering process by:
\begin{equation} \label{eq:filter}
  \mathcal{H'}_{t}^k =  \mathbf{m}_k \odot \mathbf{v}_k \mathcal{X'}_t^k = [ \overbrace{ \underbrace{1,\cdots,1}_{S/2^{(K-k)}}, 0, \cdots,0 }^{S} ] \odot \mathbf{v}_k \mathcal{X'}_t^k
\end{equation}
where $\mathbf{m}_k$ is the mask vector of length $S$ for filtering; $\mathbf{v}_k$ is a randomly-initialized vector of order $k$ which is learnable; $\mathcal{H'}_{t}^k$ are the filtered frequency representations. In particular, Equation (\ref{eq:filter}) is inspired by that the low-order derived representations include more noises than more stationary high-order representations and thus should be filtered. To filter frequencies, we design an exponential-masking mechanism for $\mathbf{m}_k$ to select $\frac{S}{2^{(K-k)}}$ frequencies while filtering others.

\paragraph{Fourier Convolutions} Given the filtered signal representations in the frequency domain, the subsequent step involves acquiring the dependencies for time series forecasting. Considering that the representations are complex value, it is intuitive to devise a network in which all operations are conducted in the frequency domain.
According to the convolution theorem~\cite{1970An}, the Fourier transform of a convolution of two signals equals the pointwise product of their Fourier transforms in the frequency domain. Thus, by conducting a straightforward product in the frequency domain, it is equivalent to perform global convolutions in the time domain which allows the capture of dependencies. 

Accordingly, we employ Fourier convolution layers that involve performing a product in the frequency domain, to capture these dependencies.
Specifically, given $\mathcal{H'}_{t}^k$ achieved by order-adaptive filtering, we compute it as follows:
\begin{equation} \label{eq:fc}
 \mathcal{H}_{t}^k = {\rm FourierConvolution(\mathcal{H'}_{t}^k)}   =  \mathcal{H'}_{t}^k \mathcal{W}_k
\end{equation}
where $\mathcal{W}_k$ is the weighted matrix to conduct the convolutions in the frequency domain; $\mathcal{H}_{t}^k$ is the output by our order-adaptive Fourier convolution network when the order is $k$.
Equation (\ref{eq:fc}) is intuitive which aims to directly learn the dependencies on the filtered components for forecasting.


\section{Experiments}
In this section, in order to evaluate the performance of our model, we conducts extensive experiments on six real-world time series benchmarks to compare with the state-of-the-art time series forecasting methods. 
\subsection{Experimental Setup}
\paragraph{Datasets} We follow previous work~\cite{autoformer21,fedformer22,PatchTST2023,frets_2023} to evaluate our \textsc{DeRiTS} on different representative datasets from various application scenarios, including Electricity~\cite{asuncion2007uci}, Traffic~\cite{autoformer21}, ETT~\cite{zhou2021informer}, Exchange~\cite{Lai2018}, ILI~\cite{autoformer21}, and Weather~\cite{autoformer21}. We preprocess all datasets following the recent frequency learning work \cite{frets_2023} to normalize the datasets and split the datasets into training, validation, and test sets by the ratio of 7:2:1. We leave more dataset details in Appendix \ref{appendix:dataset}.

\begin{table}[h]
    \centering
    \vspace{-2mm}
    \caption{Performance comparisons on MAE and RMSE with state-of-the-art normalization techniques in  time series forecasting taking LTSF-Linear as the backbone.
    }
    \label{table:results_with_fdt}
    \resizebox{\linewidth}{!}
    {
    \begin{tabular}{l c|c c|c c|c c|c c}
    
    \toprule
       \multicolumn{2}{c|}{Models}  &\multicolumn{2}{c|}{{LTSF-Linear}} &\multicolumn{2}{c|}{+RevIN} & \multicolumn{2}{c|}{+Dish-TS} & \multicolumn{2}{c}{+FDT} \\
        \multicolumn{2}{c|}{Metrics}&MAE &RMSE &MAE &RMSE &MAE &RMSE&MAE &RMSE \\
       \midrule
        \multirow{4}*{\rotatebox{90}{Exchange}} & 96 & {0.038} &  {0.052} & {0.040} &  {0.055}& 0.039 &   {0.053} & \textbf{0.036} & \textbf{0.050} \\
         & 192 &  {0.053}&  {0.069} &  {0.052}&  {0.070}&0.055& 0.071 &  \textbf{0.050}& \textbf{0.068}\\
         & 336 &  {0.064}&  {0.085} &  {0.069}&  {0.094} &0.068& 0.090 & \textbf{0.060}& \textbf{0.082}\\
         & 720 &  {0.092}&  {0.116} &  {0.115}&  {0.145}&0.110& 0.132 & \textbf{0.090}& \textbf{0.114}\\
         \midrule
         \multirow{4}*{\rotatebox{90}{Weather}} &  96 & {0.040}&  {0.081}  &  {0.042}&  {0.085}&   {0.039}&   {0.082}& \textbf{0.037} & \textbf{0.080}\\
         & 192 & {0.048}&  {0.089}  &  {0.045}&  {0.089}&   {0.046}&   {0.090} &  \textbf{0.043}& \textbf{0.088}\\
          &  336  & {0.056}&  {0.098} &  {0.053}&  {0.097}&   {0.055}&   {0.099}& \textbf{0.050} & \textbf{0.095}\\
         & 720 & {0.065} &  {0.106} & {0.061} &  {0.108}&   {0.060}&   {0.105}& \textbf{0.056} & \textbf{0.103} \\
         \midrule
         \multirow{4}*{\rotatebox{90}{ILI}} & 24 & {0.167} &  {0.214} & {0.151} &  {0.199}&0.156&   {0.203} &  \textbf{0.141} & \textbf{0.196} \\
         & 36 &  {0.179}&  {0.231} &  {0.168}&  {0.228}&0.171& 0.230 & \textbf{0.158} & \textbf{0.212}\\
         & 48 &  {0.165}&  {0.216} &  {0.158}&  {0.214} &0.160& 0.214 & \textbf{0.151} & \textbf{0.198}\\
         & 60 &  {0.166}&  {0.212} &  {0.161}&  {0.204}&0.164& 0.210 &  \textbf{0.152} & \textbf{0.196} \\
         \midrule
         \multirow{4}*{\rotatebox{90}{Electricity}} &  96 & {0.045}& {0.075}  &  {0.046}&  {0.078}&   {0.044}&   {0.074}& \textbf{0.041}&\textbf{0.072} \\
         & 192 &  {0.043}&  {0.070}  &  {0.042}&  {0.070}&   {0.043}&   {0.071}& \textbf{0.040}& \textbf{0.068}\\
          &  336 & {0.044}& {0.071} & {0.043}& {0.070}&  {0.042}&  {0.070}&  \textbf{0.040}&\textbf{0.067} \\
         & 720 & {0.054} & {0.080}  & {0.048} &  {0.076}& 0.050& {0.077}& \textbf{0.048} & \textbf{0.076}\\
         \midrule
         \multirow{4}*{\rotatebox{90}{ETTh1}} &  96 & {0.063}& {0.089}  &  {0.061}&  {0.088}& 0.062& 0.089&  \textbf{0.060}& \textbf{0.084} \\
         & 192 & {0.067}&  {0.094}  &  {0.065}&  {0.092}& 0.066&   {0.092}& \textbf{0.062}& \textbf{0.090}\\
          &  336 & {0.070}& {0.097} &  {0.068}&  {0.095}&   {0.068} & {0.096} & \textbf{0.064}& \textbf{0.092}\\
         & 720 & {0.082} & {0.108}  & {0.089} &  {0.110}&   {0.091}&   {0.111}& \textbf{0.076} & \textbf{0.100}\\
         \midrule
         \multirow{4}*{\rotatebox{90}{ETTm1}} &  96 & {0.055}&  {0.080}  &  {0.054}&  {0.078}&   {0.052} & 0.077&  \textbf{0.051} & \textbf{0.072} \\
         & 192 &  {0.060}&  {0.087}  &  {0.058}&  {0.086}&   {0.057}&   {0.085}& \textbf{0.056} & \textbf{0.084} \\
          &  336 & {0.065}& {0.093} &  {0.062}&  {0.090}&   {0.064} &   {0.092} &  \textbf{0.060} & \textbf{0.088}\\
         & 720 & {0.072} &  {0.099}  & {0.070} &  {0.100}&   {0.071}&   {0.102}& \textbf{0.066}& \textbf{0.093}\\
         \bottomrule
    \end{tabular}
    }

    \vspace{-1mm}
\end{table}

\paragraph{Baselines} We conduct a comprehensive comparison of the forecasting performance between our model \textsc{DeRiTS} and several representative and state-of-the-art (SOTA) models on the six datasets, including Transformer-based models: Informer~\cite{zhou2021informer}, Autoformer~\cite{autoformer21}, FEDformer~\cite{fedformer22}, PatchTST~\cite{PatchTST2023}; MLP-based model: LSTF-Linear~\cite{DLinear_2023}; Frequency domain-based model: FreTS~\cite{frets_2023}. Besides, we also consider the existing normalization methods towards distribution shifts in time series forecasting, including RevIN~\cite{kim2021reversible}, NSTransformer~\cite{liu2022non} and Dish-TS~\cite{fan2023dish}.
All the baselines we reproduced are implemented based on their official code and we leave more baseline details in Appendix \ref{appendix:baseline}.

\paragraph{Implementation Details} 
We conduct our experiments on a single NVIDIA RTX 3090 24GB GPU with PyTorch 1.8~\cite{PYTORCH19}. We take MSE (Mean Squared Error) as the loss function and report the results of MAE (Mean Absolute Errors) and RMSE (Root Mean Squared Errors) as the evaluation metrics. A lower MAE/RMSE indicates better performance of time series forecasting. More detailed information about the implementation are included Appendix~\ref{appendix:implementation}.

\subsection{Overall Performance}
To verify the effectiveness of \textsc{DeRiTS}, we conduct the  performance comparison of multivariate time series forecasting in several benchmark datasets. Table~\ref{tab:long_term} presents the overall forecasting performance in the metrics of MAE and RMSE under different prediction lengths. 
In brief, the experimental results demonstrate that \textsc{DeRiTS} achieves the best performances in most cases as shown in Table \ref{tab:long_term}. Quantitatively, compared with the best results of transformer-based models, \textsc{DeRiTS} has an average decrease of more than 20\% in MAE and RMSE. Compared with more recent frequency learning model, FreTS~\cite{kunyi_2023} and the state-of-the-art transformer model, PathchTST~\cite{PatchTST2023}, \textsc{DeRiTS} can still outperform them in general. This has shown the great potential of \textsc{DeRiTS} in the time series forecasting task.

\subsection{Comparison with Normalization Techniques}
In this section, we further compare our performance with the recent normalization technique, RevIN \cite{kim2022reversible} and Dish-TS~\cite{fan2023dish} that handle distribution shifts in time series forecasting. Table \ref{table:results_with_fdt} has shown the performance comparison in time series forecasting taking the LTSF-Linear~\cite{zeng2022transformers}. Since FDT transforms signals to the frequency domain, we implement a simple single-layer Linear model in the frequency domain. From the results, we can observe that the existing RevIN and Dish-TS can only improve the backbone in some shifted datasets. In some situations, it might lead to worse performances. Nevertheless, our FDT can usually achieve the best performance. A potential explanation is that FDT transforms data with full frequency spectrum and thus achieves stable improvement while other normalization techniques cannot reveal full data distribution and thus cannot make use of them for transformation.


\begin{table}[!t]
    \centering
    \caption{The impact of frequency derivative transformation with order $k$. For Exchange and Weather datasets, the prediction length and the lookback window size are 96. For ILI dataset, the prediction length and the lookback window size are 36 due to length limitation.} \label{table:fdt_analysis}
    \vspace{-2mm}
    \scalebox{0.9}{
    \begin{tabular}{c|c c | c c | c c}
    \toprule
    Datasets & \multicolumn{2}{c|}{Exchange} & \multicolumn{2}{c|}{ILI} & \multicolumn{2}{c}{Weather}\\
     Metrics    &  MAE & RMSE & MAE & RMSE & MAE & RMSE\\
     \midrule
     $k=0 $   & {0.041} & 0.058 & {0.179} & 0.231 & {0.040} & 0.099 \\ 
    \midrule
      $k=1 $  & 0.037 & 0.053 & 0.159 & 0.213 & 0.038 & 0.081\\
    \midrule
       $k=2$   & 0.035 & 0.050 & 0.157& 0.212 & 0.037 & 0.080\\
    \midrule
        $k=3$   & 0.036 & 0.052 & 0.162 & 0.216 & 0.037 & 0.081\\
    \bottomrule
    \end{tabular}
    }
    \label{tab:order_k}
    \vspace{-1mm}
\end{table}

\subsection{Model Analysis}

\paragraph{Impact of Frequency Derivative Transformation} It is notable that
our proposed FDT plays an important role in \textsc{DeRiTS}, and we aim to analysis the impact of FDT on the model performance. Thus, we consider a special case of FDT, which is when we set $k=0$, the derivation is removed and FDT is degraded to naive Fourier transform. In addition to this setting, we also vary different orders ($k$) of derivation to test the effectiveness. Table \ref{table:fdt_analysis} has shown the results on three datasets. We can easily observe that the performance of \textit{DeRiTS} can beat the variant version without the derivation, which has demonstrate the effectiveness of FDT. Moreover, with the increase of order, the original time series would be derived too much. This might cause the information loss which leads to performance degradation accordingly.

%

\begin{figure}[!h]
    \centering
    \subfigure[Exchange dataset]
    {
        \centering
        \includegraphics[width=0.32\linewidth]{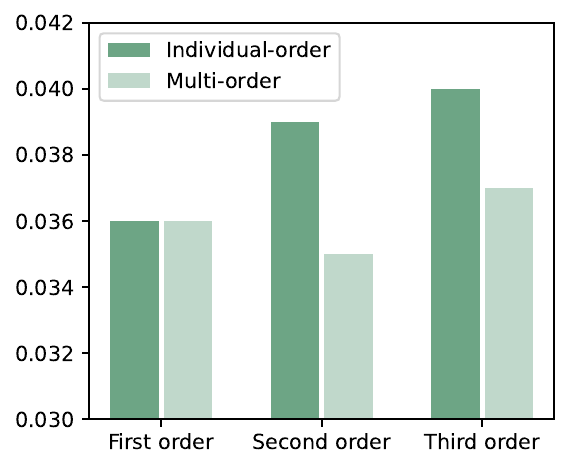}
         \label{Exchange}
    }\hspace{-3mm}
    \subfigure[Weather dataset]
    {
        \centering
        \includegraphics[width=0.32\linewidth]{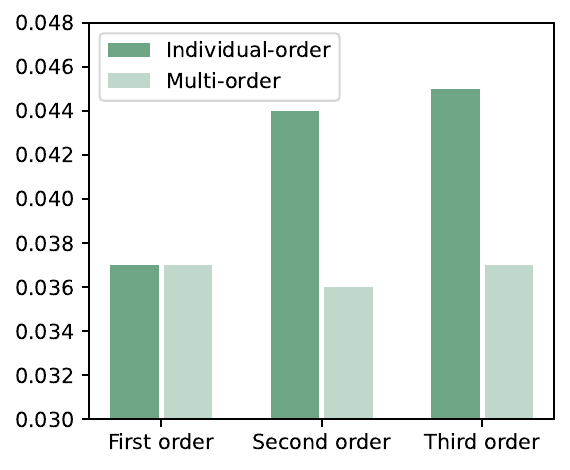}
        \label{Weather}
    }\hspace{-3mm}
    \subfigure[ILI dataset]
    {
        \centering
        \includegraphics[width=0.32\linewidth]{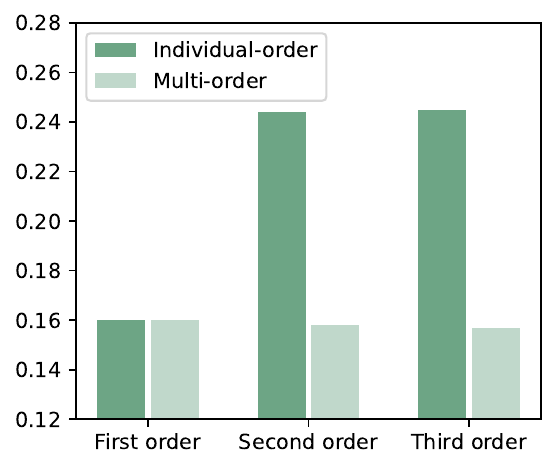} 
        \label{Weather}
    }
    \vspace{-3mm}
    \caption{The forecasting performance (MAE) comparison  between original \textsc{DeRiTS} (multi-order) and its individual-order variant. The lower values indicate the better forecasting performance.} \label{fig:arc_analysis}
    \vspace{-2mm}
\end{figure}

\paragraph{Impact of Multi-order Stacked Architecture} As aforementioned in Section \ref{sec:fre_arch}, we organize \textsc{DeRiTS} as a parallel-stacked architecture for multi-order fusion. Thus we aim to study the impact of such a stacked architecture. In contrast to multi-order stacked \textsc{DeRiTS}, we have considered another situation of \textit{individual-order} \textsc{DeRiTS} that removes the parallel-stacked architecture. Figure \ref{fig:arc_analysis} has shown the performance comparison with \textit{individual-order} \textsc{DeriTS} in the Exchange, ILI, and Weather datasets. It can be easily observed that without the parallel-stacked architecture for multi-order fusion, the \textit{individual-order} \textsc{DeriTS} achieves much worse performance than the original one even under different orders, which signifies the necessity of our multi-order design in the frequency derivative learning architecture.


\paragraph{Lookback Analysis} We aim to examine the impact of the lookback window on forecasting performance of \textsc{DeRiTS}. Figure \ref{fig:large_lookback} has demonstrated the experimental results on the Exchange, Weather, and ILI datasets. Specifically, we maintain the prediction length as 96 and and vary the lookback length from 48 to 240 on Exchange and Weather datasets. For the ILI dataset, we keep the prediction length at 36 and alter the lookback window size from 24 to 72. From the results, we can observe that in most cases, larger lookback length would bring up less prediction errors; this is because larger input includes more temporal information. Also, larger input length would also bring more noises hindering forecasting, while our \textsc{DeRiTS} can achieve comparatively stable performance.


\begin{figure}
    \centering
     \subfigure[Exchange dataset]
    {
        \centering
        \includegraphics[width=0.3\linewidth]{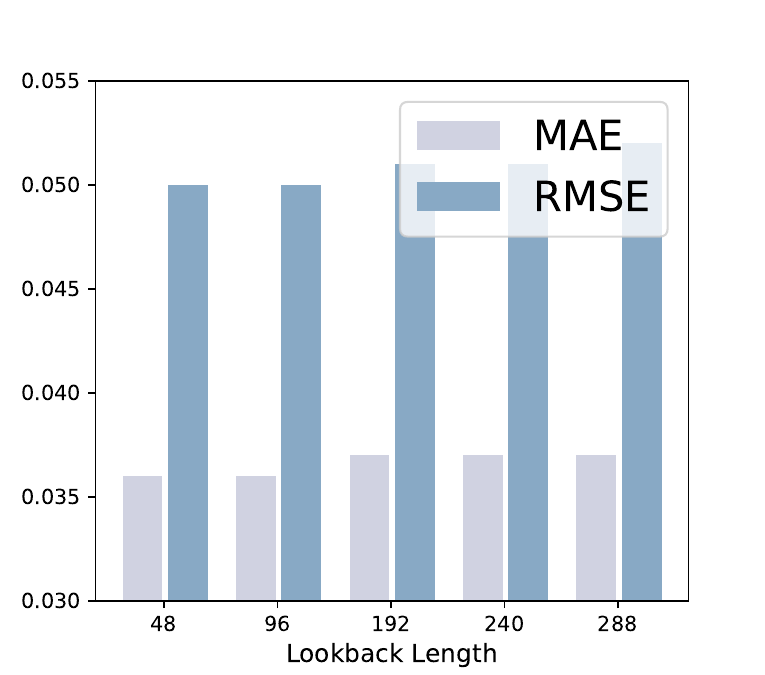}
         \label{Exchange}
    }
    \subfigure[Weather dataset]
    {
        \centering
        \includegraphics[width=0.3\linewidth]{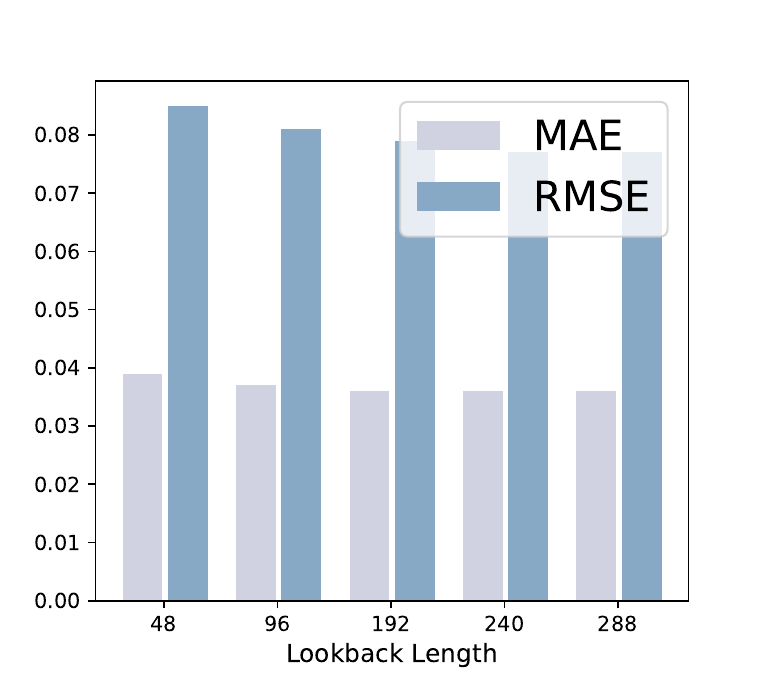}
        \label{Weather}
    }
    \subfigure[ILI dataset]
    {
        \centering
        \includegraphics[width=0.3\linewidth]{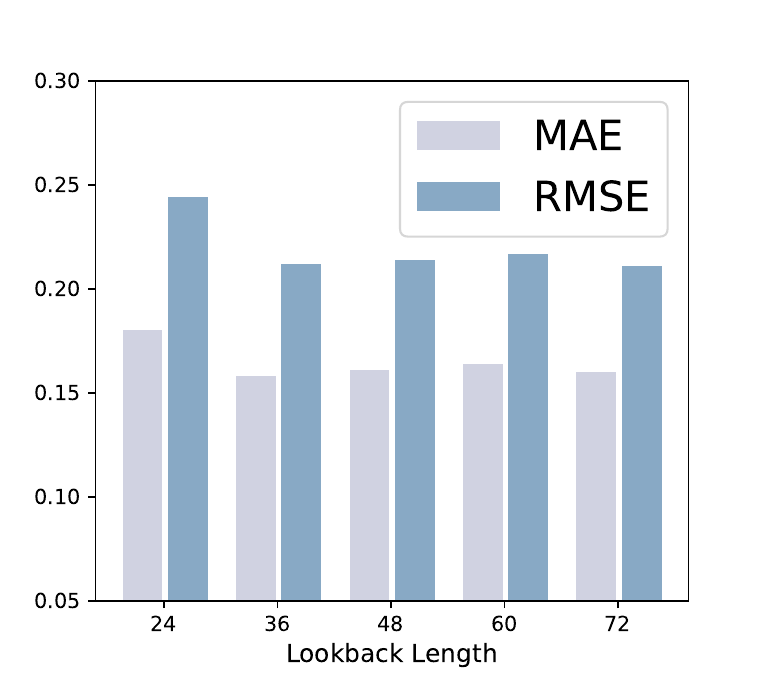}
        \label{Weather}
    }
    \vspace{-3mm}
    \caption{Impact of lookback length on forecasting. Metrics MAE and RMSE are reported with the length of lookback window prolonged and the prediction length fixed. 
    }
    \label{fig:large_lookback}
    \vspace{-2mm}
\end{figure}

\begin{table}[h]
\small
\centering
\caption{Efficiency analysis. We report the training time of \textsc{DeRiTS} and Non-Stationary  
 (NS) transformer-based methods.} 
\vspace{-2mm}
\label{table:trainingtime}
\resizebox{0.8\linewidth}{!}{
\begin{tabular}{ccccc}
\toprule
Length     & 96     & 192    & 336   & 480    \\
\midrule
NS-FEDformer  & 137.7 & 160.4 & 192.8 & 227.2 \\
NS-Autoformer & 44.41  & 59.23   & 78.29 & 101.5 \\
NS-Transformer   & 30.24  & 41.38  & 50.21 & 61.88  \\
\midrule
\textsc{DeRiTS}     & 12.57  & 13.87  & 14.93 & 15.93  \\
\bottomrule
\end{tabular}
}
\vspace{-2mm}
\end{table}

\paragraph{Efficiency Analysis} To conduct the efficiency analysis for our framework, we report the training time of \textsc{DeRiTS} across various prediction lengths, and we also include the training time of the Non-Stationary transformer~\cite{liu2022non} for comparison, coupled with its corresponding backbones such as FEDformer, Autoformer and Transformer.
The experiments are conducted under the prediction length  
with the same input length of 96 on the Exchange dataset.
As shown in Table~\ref{table:trainingtime}, the results prominently highlight that our model exhibits superior efficiency metrics. 
Our \textsc{DeRiTS} significantly reduces the number of parameters thus enhancing the computation speed. In particular, our model showcases an average speed that is several times faster than the baselines. These findings underscore the efficiency gains achieved by our model, positioning it as a compelling choice for non-stationary time series forecasting.

\subsection{Visualization Analysis}
\paragraph{FDT and Non-stationarity} To study the Fourier derivative transformation in \textsc{DeRiTS}, we visualize the original signals and derived signals for comparison. Since the direct outputs of FDT are complex values that are difficult to visualize completely, we thus transform the derived frequency components back to the time domain via inverse FDT, which allows us to show the corresponding time values for visualization. 
Specifically, we choose two non-stationary time series from the Weather dataset and Exchange dataset, respectively. As shown in Figure \ref{fig:multi-order-fdt}, we can observe the original signals have included obvious non-stationary oscillations and trends. In contrast, the derived signals exhibit a larger degree of stationarity compared with raw data. This further reveals that learning in the derivative representation of signals is more stationary and thus can achieve better performance.

\begin{figure}[!t]
    \centering
    \subfigure[Weather dataset]
    {
        \centering
        \includegraphics[width=0.46\linewidth]{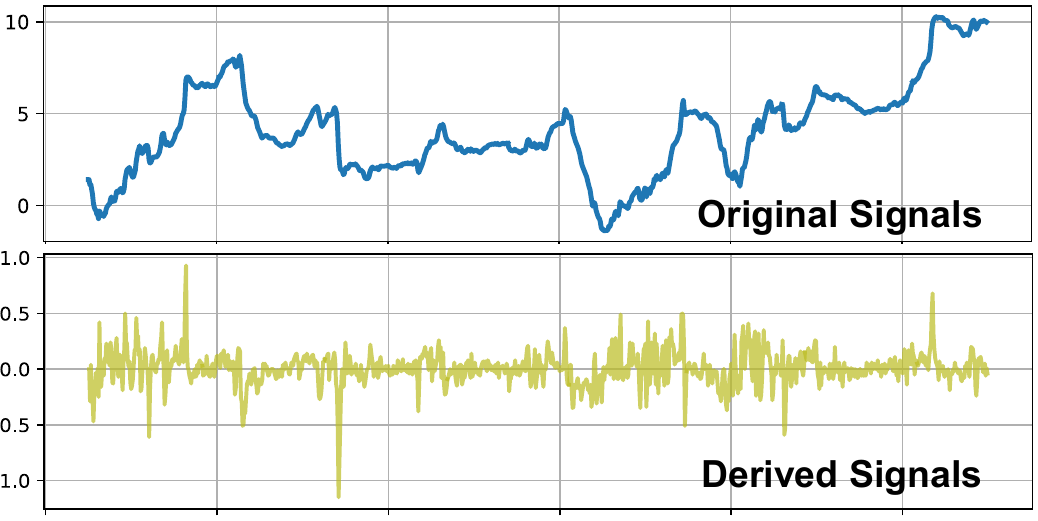}
    }
    \hspace{-2mm}
    \subfigure[{Exchange dataset}]
    {
        \centering
        \includegraphics[width=0.48\linewidth]{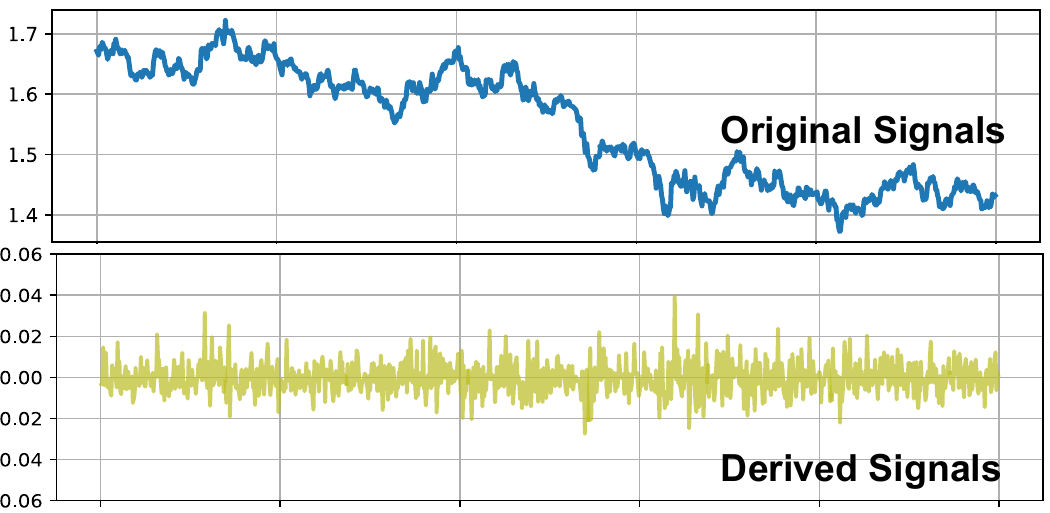}
    }
    \vspace{-3mm}
    \caption{Visualization comparison of original signals and derived signals with Fourier derivative transformation.}
    \vspace{-3mm}
    \label{fig:multi-order-fdt}
\end{figure}

    

\vspace{-2mm}
\begin{figure}[!h]
    \centering
    \subfigure[\textsc{DeRiTS}]
    {
        \centering
        \includegraphics[width=0.45\linewidth]{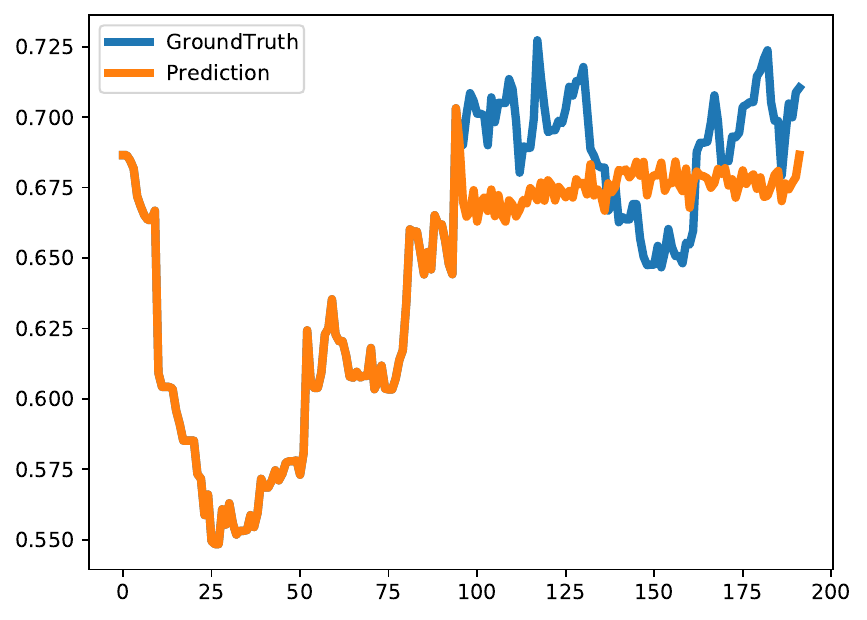}
        \label{N=24}
    }
    \subfigure[NSTransformer]
    {
        \centering
        \includegraphics[width=0.45\linewidth]{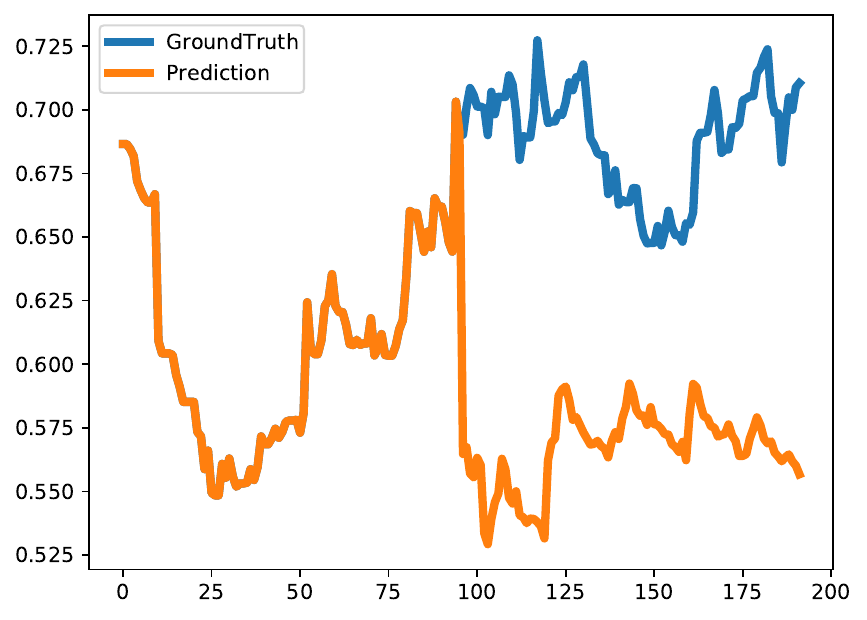}
        \label{N=29}
    }
    \vspace{-3mm}
    \caption{Visualizations of non-stationary forecasting (prediction vs. ground truth) on the Exchange dataset.}
    \label{fig:visualization}
    \vspace{-3mm}
\end{figure}

\paragraph{Case Study of Forecasting} To further analysis the model performance, we carry out the case study for non-stationary time series forecasting. Figure \ref{fig:visualization} demonstrates the visualization of forecasting results of \textsc{DeRiTS} and NSTransformer~\cite{liu2022non} with the prediction length as 96 and lookback length as 96. 
Upon careful observation of it, it becomes evident that \textsc{DeRiTS} can be capable of aligning with the ground truth when the time series distribution is largely shifted, while the baseline method deviates from the true values.
The visualizations demonstrates the model's adaptability to shifts. We include more visualizations in Appendix \ref{appendix:visualization}.

\section{Conclusion Remarks}
In this paper, we propose to address non-stationary time series forecasting from the frequency perspective. Specifically, we utilize the whole frequency spectrum for the transformation of time series in order to make full use of time series distribution. Motivated by this point, we propose a {deep frequency derivative learning} framework \textsc{DeRiTS} for non-stationary forecasting, which is mainly composed of the Frequency Derivative Transformation and the Order-adaptive Fourier Convolution Network with a parallel-stacked architecture. Extensive experiments on real-world datasets have demonstrated its superiority.
Moreover, distribution shifts and non-stationarity are actually a pervasive and crucial topic for time series forecasting. Thus we hope that the new perspective of frequency derivation together with the \textsc{DeRiTS} framework can facilitate more future related research.

\section*{Acknowledgments}
This work was partially supported by the National Natural Science Foundation of China (No. 62072153) the Anhui Provincial Key Technologies R\&D Program (No. 2022h11020015), the Shanghai Baiyulan Talent Plan Pujiang Project (23PJ1413800), and the National Natural Science Foundation of China (No. 623B2043).

\appendix
\section{Experiment Details}\label{experimental_details}
\subsection{Dataset Details} \label{appendix:dataset}
We follow previous works ~\cite{autoformer21,PatchTST2023,DLinear_2023} and adopt seven real-world datasets in the experiments to evaluate the accuracy of time series forecasting, including Exchange~\footnote{\url{https://github.com/laiguokun/multivariate-time-series-data}}
, ILI~\footnote{\url{https://gis.cdc.gov/grasp/fluview/fluportaldashboard.html}}, Weather~\footnote{\url{https://www.bgc-jena.mpg.de/wetter/}}
,  
Traffic ~\footnote{\url{http://pems.dot.ca.gov}},  
Electricity
~\footnote{\url{https://archive.ics.uci.edu/ml/datasets/ElectricityLoadDiagrams20112014}}
,  
and ETTh1\&ETTm1~\footnote{\url{https://github.com/zhouhaoyi/ETDataset}}
.  
We summarize the datasets in Table \ref{tab:datasets}.
\vspace{-3mm}

\begin{table}[!h]
    \centering
    \caption{Summary of datasets.}
    \vspace{-3mm}
    \begin{tabular}{c|c c c }
    \toprule
      Datasets   & Variables  & Samples & Granularity\\
      \midrule
      \midrule
      Exchange  &  8 & 7,588 & 1day\\
      ILI &  7 & 966 & 1week \\
      Weather & 21 & 52,696 & 10min\\
      Traffic & 862 & 17,544 & 1hour\\
      Electricity & 321 & 26,304 & 1hout\\
      ETTh1 & 7 & 17,420 & 1hour\\
      ETTm1 & 7 & 69,680 & 5min\\
      \midrule
    \bottomrule
    \end{tabular}
    \label{tab:datasets}
    \vspace{-5mm}
\end{table}

\subsection{Baselines } \label{appendix:baseline}
We compare our model \textsc{DeRiTS} with other seven time series forecasting methods, including
FreTS~\cite{frets_2023},
PatchTST~\cite{PatchTST2023},
FEDformer~\cite{fedformer22},
Autoformer~\cite{autoformer21},
Informer~\cite{zhou2021informer},
DLinear~\cite{DLinear_2023},
and NSTransformer~\cite{liu2022non}.
Also, we compare our FDT with normalization methods, including RevIN~\cite{kim2021reversible} and SAN~\cite{liu2023adaptive}.
We obtained the baseline codes from their respective official GitHub repositories. As the datasets serve as general benchmarks, we can reproduce their codes according to their recommended settings.

\subsection{Implementation} \label{appendix:implementation}
We adhere to the experimental settings outlined in FreTS~\cite{frets_2023}. For certain datasets, we meticulously fine-tune hyperparameters such as batch size and learning rate on the validation set, selecting configurations that yield optimal performance. Batch size tuning is conducted over the set \{4, 8, 16, 32\}. The default setting for the order $k$ is 2. The codes will be publicly available soon.

\section{Proof} \label{proof}
\subsection{The Equivalence of the Mean Value from a Frequency Perspective} \label{proof_mean}

For convenience, we employ the discrete representation of the signal to demonstrate the equivalence of the mean value. Given a signal $x[n]$ with a length of $N$, we can obtain its corresponding discrete Fourier transform $\mathcal{X}[f]$ by:
\begin{equation}
    \mathcal{X}[f] = \frac{1}{N}\sum_{n=0}^{N-1} x[n]e^{2\pi jnf/N}
\end{equation}
where $j$ is the imaginary unit.
We set $f$ as 0 and then,
\begin{equation}
\begin{split}
    \mathcal{X}[\textcolor{red}{0}] &= \frac{1}{N}\sum_{n=0}^{N-1} x[n]e^{2\pi jn\textcolor{red}{0}/N} \\
    &=  \frac{1}{N}\sum_{n=0}^{N-1}x[n]. 
\end{split}
\end{equation}
According the above equation, we can find that the mean value $\frac{1}{N}\sum_{n=0}^{N-1}x[n]$ in the time domain is equal to the zero frequency component $\mathcal{X}[0]$ in the frequency domain.

\begin{table*}[ht]
    \centering
    \caption{Long-term forecasting results comparison with different lookback window lengths $L \in \{36, 72, 108\}$ on the ILI dataset. The prediction lengths are as $H \in \{24, 36, 48, 60\}$.
    The best results are in \textbf{bold} and the second best results are \underline{underlined}.}
    \scalebox{0.78}{
    \begin{tabular}{l c|c c|c c|c c|c c| c c|c c |c c |c c}
    \toprule
       \multicolumn{2}{c|}{Models} &\multicolumn{2}{c|}{\textsc{DeRiTS}} &\multicolumn{2}{c|}{FreTS} & \multicolumn{2}{c|}{PatchTST} & \multicolumn{2}{c|}{LTSF-Linear} & \multicolumn{2}{c|}{FEDformer} & \multicolumn{2}{c|}{Autoformer} & \multicolumn{2}{c|}{Informer} & \multicolumn{2}{c}{NSTransformer}\\
        \multicolumn{2}{c|}{Metrics}&MAE &RMSE &MAE &RMSE &MAE &RMSE&MAE &RMSE&MAE &RMSE &MAE &RMSE &MAE &RMSE &MAE &RMSE\\
       \midrule
         \multirow{4}*{\rotatebox{90}{36}} & 24 & \textbf{0.141}& 0.197 &\underline{0.143} & \textbf{0.192}&\underline{0.143}& \underline{0.196} &{0.167} & {0.214}& 0.195& 0.246& 0.208& 0.260& 0.192& 0.259& 0.166&0.228 \\
         & 36 & \textbf{0.158} & \textbf{0.212}& \underline{0.166}& \underline{0.222}&0.182& 0.239 & {0.179}& {0.231}& 0.182& 0.246& 0.190& 0.255& 0.233& 0.303& 0.187& 0.254\\
         & 48 & \textbf{0.150} & \textbf{0.198} & {0.166}& {0.226}&\underline{0.159}& \underline{0.213} & {0.165}& {0.216}& 0.173& 0.231& 0.178& 0.238& 0.214& 0.279& 0.172& 0.235\\
         & 60 & \textbf{0.152} & \textbf{0.196} & {0.166}& {0.221}&\underline{0.161}& \underline{0.209} & {0.166}& {0.212}& 0.167& 0.218& 0.171& 0.224& 0.208& 0.272& 0.164& 0.220\\
         \midrule
         \multirow{4}*{\rotatebox{90}{72}} &  24 & \textbf{0.138} & \textbf{0.187}& {0.142}& {0.193}& \underline{0.139} & \underline{0.188} &0.152 & {0.197}& 0.178& 0.226& 0.196& 0.246& 0.193&0.259& 0.154 & 0.200 \\
         & 36 &\textbf{0.160} & \textbf{0.211}& {0.173}& {0.228}& {0.173}& {0.229} &0.174 & {0.224}& 0.196& 0.259& 0.196 & 0.258&0.233& 0.301&\underline{0.168}& \underline{0.222}\\
          &  48 & \underline{0.152}& \textbf{0.198}& \textbf{0.150}& \underline{0.202}& {0.163} & {0.214} & {0.160}& {0.207} & 0.184& 0.243&0.184& 0.240& 0.214& 0.282& 0.163& 0.210\\
         & 60 & \textbf{0.154} & \textbf{0.200} & \underline{0.161} & \underline{0.209}& {0.165} & {0.213} & {0.161}& {0.206}& 0.177& 0.232& 0.175& 0.229& 0.200&0.264 & 0.164& 0.212\\
         \midrule
         \multirow{4}*{\rotatebox{90}{108}} &  24 & \underline{0.122}& \underline{0.163}& \textbf{0.120}& \textbf{0.154}& {0.136}& {0.180}& 0.141 & 0.179 & 0.178&0.222 & 0.185& 0.231& 0.196 & 0.260& 0.154& 0.194\\
         & 36 & \textbf{0.136}& \underline{0.179}& \underline{0.137}& \textbf{0.169}& {0.158}& {0.206}& {0.156}& 0.197& 0.197& 0.253& 0.198& 0.250& 0.244& 0.314& 0.150& 0.189\\
          &  48 & \textbf{0.134}& \textbf{0.173}& {0.147}& \underline{0.176}& {0.164}& {0.211}& \underline{0.144}& {0.184}& 0.188&0.241 & 0.188& 0.241& 0.217& 0.287& 0.164& 0.206\\
         & 60 &\textbf{0.142} & \textbf{0.182}&\underline{0.153} & \underline{0.188}& 0.175& {0.221}& {0.154}& {0.194}&0.184 & {0.231}&0.182& 0.227& 0.212& 0.282&0.175 &0.212 \\
         \bottomrule
    \end{tabular}
    }
    \label{tab:long_term_appendix_ili}
\end{table*}

\begin{table*}[!h]
    \centering
    \caption{Long-term forecasting results comparison with different lookback window lengths $L \in \{96, 192, 336\}$ on the Exchange dataset. The prediction lengths are as $H \in \{96, 192, 336, 720\}$.
    The best results are in \textbf{bold} and the second best results are \underline{underlined}. '-' denotes out of memory. }
    \scalebox{0.78}{
    \begin{tabular}{l c|c c|c c|c c|c c| c c|c c |c c |c c}
    \toprule
       \multicolumn{2}{c|}{Models} &\multicolumn{2}{c|}{\textsc{DeRiTS}} &\multicolumn{2}{c|}{FreTS} & \multicolumn{2}{c|}{PatchTST} & \multicolumn{2}{c|}{LTSF-Linear} & \multicolumn{2}{c|}{FEDformer} & \multicolumn{2}{c|}{Autoformer} & \multicolumn{2}{c|}{Informer} & \multicolumn{2}{c}{NSTransformer}\\
        \multicolumn{2}{c|}{Metrics}&MAE &RMSE &MAE &RMSE &MAE &RMSE&MAE &RMSE&MAE &RMSE &MAE &RMSE &MAE &RMSE &MAE &RMSE\\
       \midrule
         \multirow{4}*{\rotatebox{90}{96}} & 96 & \textbf{0.035}& \textbf{0.050}& \underline{0.037} &  \underline{0.051}&0.039&  {0.052} & {0.038} &  {0.052}& 0.050& 0.067& 0.050& 0.066& 0.066& 0.084& 0.052&0.068 \\
         & 192 & \textbf{0.050} & \textbf{0.066} &  \textbf{0.050}&  \underline{0.067}&0.055& 0.074 &  \underline{0.053}&  {0.069}& 0.064& 0.082& 0.063& 0.083& 0.068& 0.088& 0.062& 0.082\\
         & 336 & \textbf{0.060} & {0.083} &  \underline{0.062}&  \underline{0.082}&0.071& 0.093 &  {0.064}&  \textbf{0.080}& 0.080& 0.105& 0.075& 0.101& 0.093& 0.127& 0.077& 0.098\\
         & 720 & \textbf{0.086} & \textbf{0.108} &  \underline{0.088}&  \underline{0.110}&0.132& 0.166 &  {0.092}&  {0.116}& 0.151& 0.183& 0.150& 0.181& 0.117& 0.170& 0.140& 0.172\\
         \midrule
         \multirow{4}*{\rotatebox{90}{192}} &  96 & \textbf{0.036} & \textbf{0.050} &  \textbf{0.036}&  \textbf{0.050}&  \underline{0.037} &  \underline{0.051} & 0.038 &  \underline{0.051}& 0.067& 0.086& 0.066& 0.085& 0.109& 0.131& 0.047 & 0.063 \\
         & 192 & \textbf{0.051} & \underline{0.070} &  \textbf{0.051}&  \textbf{0.068}&  {0.052}&  \underline{0.070} &0.053 &  \underline{0.070} & 0.080& 0.101& 0.080& 0.102& 0.144& 0.172&0.065& 0.088\\
          &  336 & \underline{0.070} & \underline{0.095} &  \textbf{0.066}&  \textbf{0.087}&  {0.072} & {0.097} & {0.073}&  {0.096} & 0.093& 0.122& 0.099& 0.129& 0.141& 0.177& 0.077& 0.103\\
         & 720 & \textbf{0.086}& \textbf{0.108}& \underline{0.088} &  \underline{0.110}& {0.099} & {0.128} &  {0.098}&  {0.122}& 0.190& 0.222& 0.191& 0.224& 0.173&0.210 & 0.142& 0.182\\
         \midrule
         \multirow{4}*{\rotatebox{90}{336}} &  96 & \textbf{0.037} & \textbf{0.051} &  \underline{0.038}&  \underline{0.052}&  {0.039}&  {0.053}& 0.040& 0.055& 0.088&0.113 & 0.088& 0.110& 0.137 & 0.169& -& -\\
         & 192 & \textbf{0.052}& \underline{0.071}&  \underline{0.053}&  \textbf{0.070}&  {0.055}&  {0.071}&  {0.055}& 0.072& 0.103& 0.133& 0.104& 0.133& 0.161& 0.195& -& -\\
          &  336 & \textbf{0.070}& \underline{0.094}&  \underline{0.071}&  \textbf{0.092}&  {0.074}&  {0.099}& {0.077}& {0.100}& 0.123&0.155 & 0.127& 0.159& 0.156& 0.193& -& -\\
         & 720 & \textbf{0.080} & \underline{0.109} & \underline{0.082} &  \textbf{0.108}& 0.100& {0.129}&  {0.087}&  {0.110}&0.210 & {0.242}& 0.211& 0.244& 0.173& 0.210&- &- \\
         \bottomrule
    \end{tabular}
    }
    \label{tab:long_term_appendix_exchange}
\end{table*}

\subsection{Proof of Proposition 1} \label{proof_fdt}
\begin{appendixproposition}
Given $X(t)$ in the time domain and $\mathcal{X}(f)$ in the frequency domain correspondingly, the $k$-order Fourier Derivative Operator on $\mathcal{X}(f)$ is equivalent to $k$-order derivation on $X(t)$ with respect to $t$ in the time domain, written by:
\begin{equation}
    (j2\pi f)^k\mathcal{X}(f) = \mathcal{F}(\frac{\mathrm{d}^k X(t)}{\mathrm{d} t^k} ),
\end{equation}
where $\mathcal{F}$ is Fourier transform, $\frac{\mathrm{d}^k}{\mathrm{d} t^k}$ is $k$-order derivative with respect to $t$, and $j$ is the imaginary unit.
\end{appendixproposition}

\begin{proof}
\renewcommand{\qedsymbol}{}
We can get $X(t)$ by the inverse Fourier transform,
    \begin{equation}
        X(t) = \frac{1}{2\pi} \int_{-\infty}^{\infty} \mathcal{X}(f) e^{j2\pi f t} df.
    \end{equation}
Then, we conduct derivation of both sides of the above equation with respect to t, 
    \begin{equation}
    \begin{aligned}
        \frac{\mathrm{d} X(t)}{\mathrm{d} t} =& \frac{1}{2\pi} \int_{-\infty}^{\infty} \frac{\mathrm{d} (\mathcal{X}(f) e^{j2\pi f t})}{\mathrm{d} t} df \\
        =& \frac{1}{2\pi} \int_{-\infty}^{\infty} ((j2\pi f) \mathcal{X}(f)) e^{j2\pi f t} df \\
        =& \mathcal{F}^{-1}((j2\pi f) \mathcal{X}(f)).
    \end{aligned}
    \end{equation}
Again, we continue conducting derivation of both sides of the above equation with respect to t, 
    \begin{equation}
       \begin{aligned}
        \frac{\mathrm{d}^2 X(t)}{\mathrm{d} t^2} =& \frac{1}{2\pi} \int_{-\infty}^{\infty} \frac{\mathrm{d} ((j2\pi f) \mathcal{X}(f) e^{j2\pi f t})}{\mathrm{d} t} df \\
        =& \frac{1}{2\pi} \int_{-\infty}^{\infty} ((j2\pi f)^2 \mathcal{X}(f)) e^{j2\pi f t} df \\
        =& \mathcal{F}^{-1}((j2\pi f)^2 \mathcal{X}(f)).
    \end{aligned} 
    \end{equation}
    By analogy, we can get
    \begin{equation}
         \frac{\mathrm{d}^k X(t)}{\mathrm{d} t^k} = \mathcal{F}^{-1}((j2\pi f)^k \mathcal{X}(f)).
    \end{equation}
    Proved.
\end{proof}
\section{Additional Results} \label{appendix:results}
To further assess our model's performance under various lookback window lengths, we conduct additional experiments on both the ILI dataset and the Exchange dataset. Specifically, for the ILI dataset, we select lookback window lengths $L$ from the set $\{36,72,108\}$ due to the limited sample lengths (refer to Table \ref{tab:datasets}). For the Exchange dataset, we opt for lookback window lengths $L$ from the set $\{96,192,336\}$. The corresponding results are illustrated in Table \ref{tab:long_term_appendix_ili} and Table \ref{tab:long_term_appendix_exchange}, respectively. From these tables, it is evident that our model, \textsc{DeRiTS}, consistently achieves strong performance across different lookback window lengths.

\section{Model Analysis} \label{appendix:analysis}
As stated in Section \ref{sec:method_fdt}, our Fourier Derivative Transformation (FDT) with its inverse can enhance the models' ability of handling non-stationarity. Specifically, supposing univariate time series signals $x_t = f(t)$ includes the periodic part $f_p(t) = cos(at+b)$ and the trend part $f^k_r(t) = c_1 t + c_2 t^2 + \cdots + c_k t^k$. Given $t$ and $t'_1$, the raw distribution shift can be seen as the difference of the two segments by $DS(t,t') = |\sum_{t=t-L}^{t} f(t) - \sum_{t=t'-L}^{t'_1} f(t)|$. For simplicity, we consider the situation when $L=1$, the distribution shifts are 
$DS(t,t') = |f(t) - f(t')|$. 
We now analyze the time domain derivation of towards the distribution shifts. 
With the derivation, $DS_d(t,t')= |\frac{\mathrm{d}f}{\mathrm{d}t}(t) - \frac{\mathrm{d}f}{\mathrm{d}t}(t') |$. 
Since periodic signals are stationary signals, we focus on the trend signals.
Supposing time series consists only trends, we have  $DS_d(t,t')= |\frac{\mathrm{d}f_r^k}{\mathrm{d}t}(t) - \frac{\mathrm{d}f_r^k}{\mathrm{d}t}(t') | = |2c_2(t - t') + \cdots k c_k (t^{k-1} - t'^{k-1}) | \leq   DS(t,t') = | c_1(t-t') + c_2(t^2 - t'^2) + \cdots + c_k(t^k - t'^k) |$  when $|t-t'| \geq 1$; hence the less shifts.
Based on the proof shown in Appendix \ref{proof}, the derivation of the time domain is equivalent to the frequency domain derivation;   thus FDT also relieves distribution shifts.



\newpage
\section{Visualization} \label{appendix:visualization}
We perform additional visualization experiments to compare our model with FreTS under various experimental settings. The results are presented in Figure \ref{fig:ili-visualization-1} and Figure \ref{fig:ili-visualization-2}. Observing these figures, it becomes apparent that our model consistently aligns well with the ground truth, even when the time series distribution undergoes substantial shifts.

\begin{figure}[h]
    \centering
    \subfigure[\textsc{DeRiTS}]
    {
        \centering
        \includegraphics[width=0.45\linewidth]{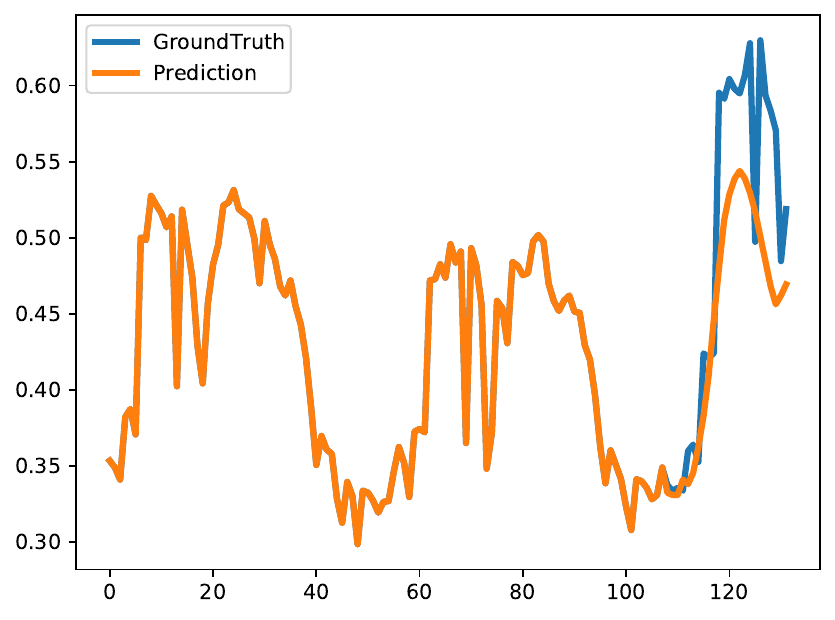}
        \label{N=24}
    }
    \subfigure[FreTS]
    {
        \centering
        \includegraphics[width=0.45\linewidth]{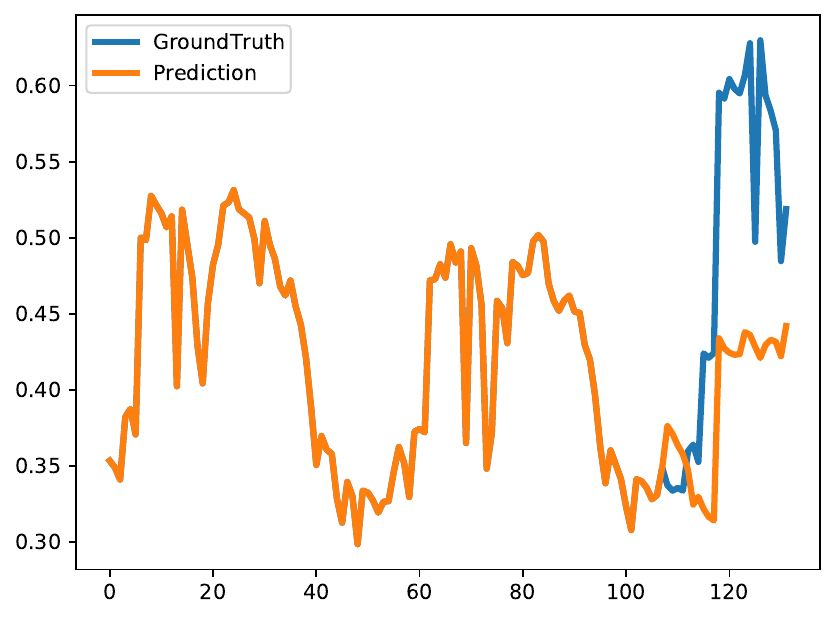}
        \label{N=29}
    }
    \vspace{-2mm}
    \caption{Visualizations of non-stationary forecasting (prediction vs. ground truth) on the ILI dataset with the lookback window length of 108 and a prediction length of 24.}
    \label{fig:ili-visualization-1}
    \vspace{-3mm}
\end{figure}

\begin{figure}[!h]
    \centering
    \subfigure[\textsc{DeRiTS}]
    {
        \centering
        \includegraphics[width=0.45\linewidth]{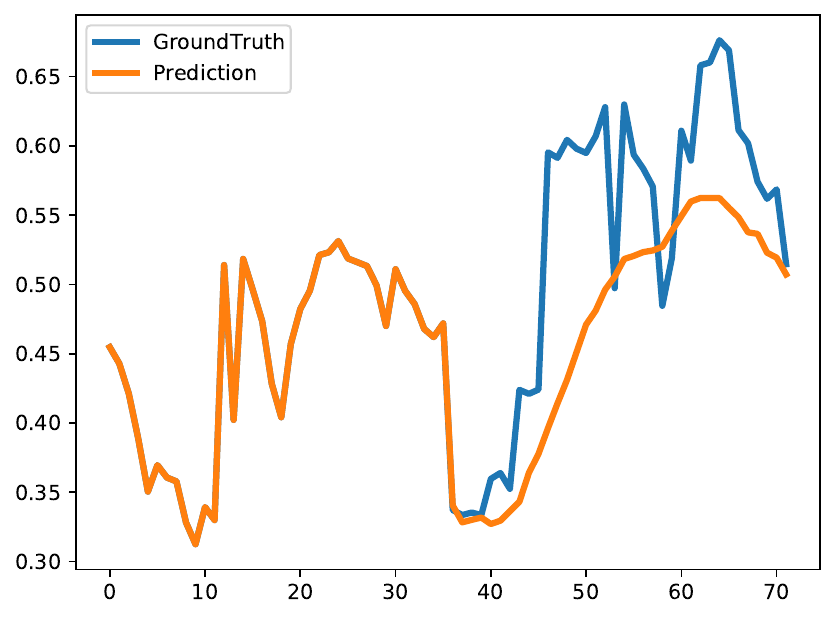}
        \label{N=24}
    }
    \subfigure[FreTS]
    {
        \centering
        \includegraphics[width=0.45\linewidth]{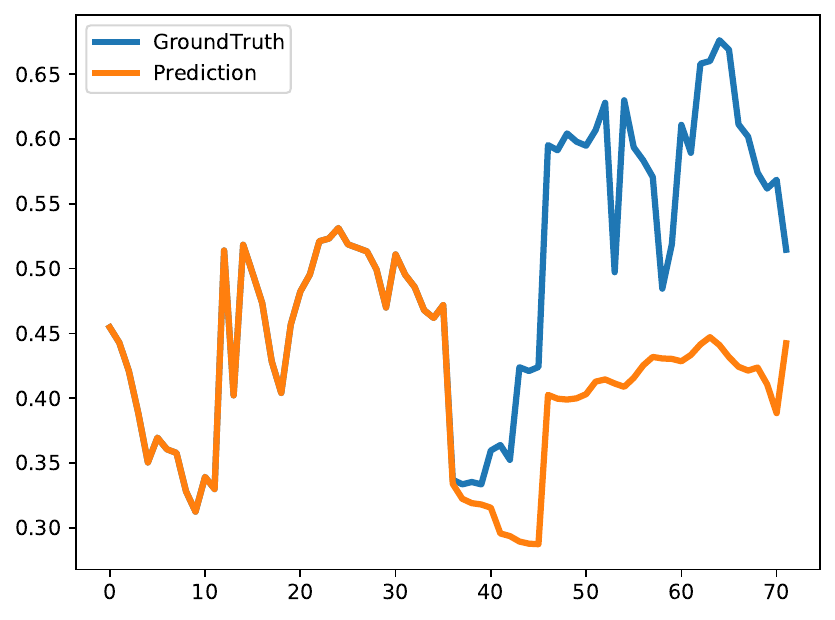}
        \label{N=29}
    }
    \vspace{-2mm}
    \caption{Visualizations of non-stationary forecasting (prediction vs. ground truth) on the ILI dataset with the lookback window length of 36 and a prediction length of 36.}
    \label{fig:ili-visualization-2}
    \vspace{-3mm}
\end{figure}

\clearpage
\bibliographystyle{named}
\bibliography{ijcai24}

\end{document}